%% file: example_paper.tex
\documentclass{article}

\usepackage{microtype}
\usepackage{graphicx}
\usepackage{subfigure}
\usepackage{booktabs} 

\usepackage{hyperref}

\usepackage[accepted]{icml2023}

\usepackage{amsmath}
\usepackage{amssymb}
\usepackage{mathtools}
\usepackage{amsthm}
\usepackage{amsfonts}
\usepackage{dsfont}
\usepackage{mathtools}
\usepackage{xcolor}
\usepackage{color,soul}
\usepackage{thmtools}
\usepackage{thm-restate}

\usepackage[capitalize,noabbrev]{cleveref}

\theoremstyle{plain}
\newtheorem{theorem}{Theorem}[section]

\newtheorem{lemma}[theorem]{Lemma}

\theoremstyle{definition}
\newtheorem{definition}[theorem]{Definition}
\newtheorem{assumption}[theorem]{Assumption}
\theoremstyle{remark}

\crefname{property}{property}{Property}
\creflabelformat{property}{(#1)#2#3}
\crefname{equation}{eq}{Eq}
\creflabelformat{equation}{(#1)#2#3}

\def\Sreset{\mathcal{S}_{reset}}
\newcommand{\Dset}[1]{\Delta_{#1}(K)}
\def\Dfree{\Dset{0}}

\newcommand{\Lagr}{\mathcal{L}}

\usepackage[textsize=tiny]{todonotes}

\icmltitlerunning{Provable Reset-free Reinforcement Learning by No-Regret Reduction}

\begin{document}

\twocolumn[
\icmltitle{Provable Reset-free Reinforcement Learning by No-Regret Reduction}



\icmlsetsymbol{equal}{*}

\begin{icmlauthorlist}
\icmlauthor{Hoai-An Nguyen}{xxx,yyy}
\icmlauthor{Ching-An Cheng}{xxx}
\end{icmlauthorlist}

\icmlaffiliation{yyy}{Rutgers University}
\icmlaffiliation{xxx}{Microsoft Research}

\icmlcorrespondingauthor{Hoai-An Nguyen}{hoaian.nguyen@rutgers.edu}
\icmlcorrespondingauthor{Ching-An Cheng}{chinganc@microsoft.com}

\icmlkeywords{Machine Learning, ICML}

\vskip 0.3in
]



\printAffiliationsAndNotice{}

\begin{abstract}
Reinforcement learning (RL) so far has limited real-world applications. One key challenge is that typical RL algorithms heavily rely on a reset mechanism to sample proper initial states; these reset mechanisms, in practice, are expensive to implement due to the need for human intervention or heavily engineered environments. To make learning more practical, we propose a generic no-regret reduction to systematically design reset-free RL algorithms. Our reduction  turns the reset-free RL problem  into a two-player game. We show that achieving sublinear regret in this two-player game would imply learning a policy that has both sublinear performance regret and sublinear total number of resets in the original RL problem. This means that the agent eventually learns to perform optimally and avoid resets. To demonstrate the effectiveness of this reduction, we design an instantiation for linear Markov decision processes, which is the first provably correct reset-free RL algorithm. 
\end{abstract}

\input{introduction}
\input{problemformulation.tex}

\input{algorithm.tex}
\input{conclusion.tex}

\section*{Acknowledgements}
Part of this work was done during Hoai-An Nguyen's internship at Microsoft Research.

\pagebreak

\bibliography{example_paper}
\bibliographystyle{icml2023}

\newpage
\appendix
\onecolumn
\input{appen.tex}


\end{document}

%% file: introduction.tex
\section{Introduction} \label{sec:intro}
Reinforcement learning (RL) enables an artificial agent to learn problem-solving skills directly through interactions. However, RL is notorious for its sample inefficiency. Successful stories of RL so far are mostly limited to applications where a fast and accurate simulator of the world is available for collecting large amounts of samples (like in games). Real-world RL, such as robot learning, remains a challenging open question. 

One key obstacle to scaling up data collection in real-world RL problems is the need for resetting the agent. The ability to reset the agent to proper initial states plays an important role in typical RL algorithms, as it affects which region the agent can explore and whether the agent can recover from its past mistakes~\citep{kakade2002approximately}. 
In most settings, completely avoiding resets without prior knowledge of the reset states or environment is impossible. 
In the absence of a reset mechanism, agents may get stuck in absorbing states, such as those where it has damaged itself or irreparably altered the learning environment.
For instance, a robot learning to walk would inevitably fall before perfecting the skill, and timely intervention is needed to prevent damaging the hardware and to return the robot to a walkable configuration. Another example is a robot manipulator learning to stack three blocks on top of each other. Unrecoverable states 
would include the robot knocking a block off the table, or the robot smashing itself forcefully into the table. Reset would then reconfigure the scene to a meaningful initial state that is good for the robot to learn from. 

Although resetting is necessary to real-world RL,
it is non-trivial to implement. Unlike in simulation, a real-world agent (e.g., a robot) cannot be reset to an arbitrary initial state with a click of a button. Resetting in the real world is expensive as it usually requires constant human monitoring and intervention. Normally, a person would need to oversee the entire learning process and manually reset the agent (e.g., a robot) to a meaningful starting state before it enters an unrecoverable state where the problem can no longer be solved. Sometimes automatic resetting can be implemented by cleverly engineering the physical learning environment~\citep{gupta21}, but it is not always feasible.

An approach we can take to make real-world RL more cost-efficient is through \emph{reset-free} RL. 
The goal of reset-free RL is not to completely remove resets, but to have an agent learn how to perform well while minimizing the amount of external resets required. Some examples of problems that have been approached in reset-free RL include agents learning dexterity skills, such as picking up an item or inserting a pipe, and learning how to walk \citep{ha20, gupta21}. While there has been numerous works proposing reset-free RL algorithms using approaches such as multi-task learning \citep{gupta21, ha20}, learning a reset policy \citep{eysenbach18, sharma22}, and skill-space planning \citep{lu20}, to our knowledge, there has not been any work with provable guarantees. 

In this work, we take the first step by providing a provably correct framework to design reset-free RL algorithms. Our framework is based on the idea of a no-regret reduction. First, we reduce the reset-free RL problem to a sequence of constrained Markov decision processes (CMDPs) with an adaptive initial state sequence. Using the special structure of reset-free RL, we establish the existence of a common saddle-point for these CMDPs. Interestingly, we show such a saddle-point exists without using the typical Slater's condition for strong duality and despite the fact that CMDPs with different initial states in general do not share a common Markovian optimal policy. Then, we derive our main no-regret reduction, which further turns this sequence into a two-player game between a primal player (updating the Markovian policy) and a dual player (updating the Lagrange multiplier function of the CMDPs) to solve for the common saddle-point. We show that if no regret is achieved in this game, then the regret of the original RL problem and the total number of required resets are both provably sublinear. This means the agent eventually learns to perform optimally and avoid resets. Using this reduction, we design a reset-free RL algorithm instantiation under the linear MDP assumption, using learning with upper confidence bound as the baseline algorithm for the primal player and projected gradient descent for the dual player. Our algorithm achieves $\tilde{O}(\sqrt{d^3H^4K})$ regret and $\tilde{O}(\sqrt{d^3H^4K})$ resets with high probability, where $d$ is the feature dimension, $H$ is the length of an episode, and $K$ is the total number of episodes.\footnote{
In the tabular MDP setting, our bounds on the regret and total resets become $\tilde{O}(\sqrt{|\mathcal{S}|^3|\mathcal{A}|^3H^4K})$, where $|\mathcal{S}|$,$|\mathcal{A}|$ denote the size of the state and action spaces, respectively.} 
\vspace{-5mm}
\section{Related Work} \label{sec:related}

Reset-free RL, despite being a promising avenue to tackle real-world RL, is a relatively new concept in the literature. The work thus far, to our knowledge, has been limited to approaches without theoretical guarantees but with only empirical verification. One such work takes the approach of learning a reset policy 
\citep{eysenbach18, sharma22}. The idea is to learn two policies concurrently: one to maximize reward, and one to bring the agent back to a reset-free initial state if they encounter a \emph{reset state} (a state which normally requires human intervention). This approach prevents the need for manual resets; however, it requires the knowledge of the reset states \citep{eysenbach18}. \citet{sharma22} avoid this assumption but assume given demonstrations on how to accomplish the goal and a fixed initial state distribution. 

Using multi-task learning is another way to perform resets without human intervention. Here, the agent learns to solve multiple tasks instead of just maximizing the reward of one task.
The hope is that a combination of the learned tasks can achieve the main goal, 
and some tasks can perform natural resets for others. 
This approach breaks down the reset process and (possibly) makes it easier to learn. However, the order in which tasks should be learned needs to be provided manually~\citep{gupta21, ha20}. 

A related problem is infinite-horizon non-episodic RL with provable guarantees (see \citet{wei20, wei19, dong19} and the references within) as this problem is also motivated by not using resets. In this setting, there is only one episode that goes on indefinitely. The objective is to maximize cumulative reward, and progress is usually measured in terms of regret with the comparator being an optimal policy. However, compared with the reset-free RL setting we study here, extra assumptions, such as the absence or knowledge of absorbing states, are usually required to achieve sublinear regret. In addition, the objective does not necessarily lead to a minimization of resets as the agent can leverage reset transitions to maximize reward. In reset-free RL, minimizing/avoiding resets is a priority. Learning in infinite-horizon CMDPs (where one possible constraint could be minimizing resets) has been studied \citep{zheng20, jain22}, but to our knowledge, all such works make strong assumptions such as a fixed initial state distribution or known dynamics. None of these assumptions are feasible for most real-world RL settings. Enforcing a fixed initial state distribution or removing absorbing states in theory oftentimes requires physically resetting a real-world agent to satisfy those desired mathematical conditions. In this paper, we focus on an episodic setting of reset-free RL (see \cref{sec:prelim}); a non-episodic formulation of reset-free RL could be an interesting one for further research.

Another related problem is safe RL, which involves solving the standard RL problem while adhering to some safety constraints. We can think of reset states in reset-free RL as unsafe states in safe RL.
There has been a lot of work in safe RL, with approaches such as utilizing a baseline safe (but not optimal) policy \citep{huang22, polo11}, pessimism \citep{amani22}, and shielding \citep{alshiekh17,wagener2021safe}. These works have promising empirical results but usually require extra assumptions such as a given baseline policy or knowledge of unsafe states. There are also provable safe RL algorithms. To our knowledge, all involve framing safe RL as a CMDP. Here, the safety constraints are modeled as a cost, and the overall goal is to maximize performance while keeping the cost below a threshold. Some of the works explicitly study safe RL while others study learning in CMDPs more generally. They commonly have provable guarantees of either sublinear regret and constraint violations, or sublinear regret with zero constraint violation \citep{wei21, hasanzadeZonuzy20, qiu20, wachi20, efroni20, ghosh22, ding20}. However, most works (including all the aforementioned ones), consider the episodic case where the initial state distribution of each episode is \emph{fixed}. This prevents a natural extension to reset-free learning as human intervention would be required to reset the environment at the end of each episode. In technical terms, this is the difference between solving a sequence of the \emph{same} CMDP versus solving a sequence of \emph{different} CMDPs. Works that allow for arbitrary initial states require fairly strong assumptions, such as knowledge (and the existence) of safe actions from each state \citep{amani21}. In our work, we utilize some techniques from provable safe RL for reset-free RL. However, it is important to note that safe RL and reset-free RL are fundamentally different, albeit related, problems. Safe RL aims to ensure the safety of an agent. On the other hand, reset-free RL aims to avoid reset states, which encompass not only unsafe states but all undesirable ones, for a varying sequence of initial states as the agent cannot be reset freely.

We weaken typical assumptions of current approaches in empirical reset-free RL, infinite-horizon RL, and safe RL by dropping any requirements on knowledge of undesirable states or for demonstrations, and by allowing \emph{arbitrary} initial state sequences that admit reset-free policies. 
Considering arbitrary initial state sequences where initial states potentially are correlated with past behaviors is not only necessary to the reset-free RL setting, but also allows for extensions to both lifelong and multi-task learning.
We achieve this important relaxation on the initial state sequence with a key observation that identifies a shared Markovian policy saddle-point across CMDPs where the constraint imposes \emph{zero} resets. This observation is new to our knowledge, and it is derived from the particular structure of  reset-free RL; we note that generally, CMDPs with different initial states do not admit shared Markovian policy saddle-points.
Finally, by the analogy between safe states and reset states, on the technical side, our framework and algorithm can also be viewed as the first provable safe RL algorithm that allows for arbitrary initial state sequences without strong assumptions.

While our main contribution is a generic reduction technique to design reset-free RL algorithms, we also instantiate the framework and achieve regret and constraint violation bounds that are still comparable to the above works when specialized to their setting. Under the linear MDP assumption, our algorithm achieves $\tilde{O}(\sqrt{d^3H^4K})$ regret and violation (equivalently, the number of resets in reset-free RL), which is asymptotically equivalent to \citet{ghosh22} and comparable to the bounds of $\tilde{O}(\sqrt{d^2H^6K})$ from \citet{ding20} for a fixed initial state.

In summary, our contributions are as follows. 
\begin{enumerate}
    \item We create a framework to design provably correct reset-free RL algorithms via a reduction first to a sequence of CMDPs with an adaptive initial state sequence, and then to a two-player game. We prove that achieving sublinear regret in this two-player game implies learning a policy that achieves sublinear performance regret and sublinear number of resets in the original problem.
    \item On the technical side, we show that such a reduction can be constructed without using the typical Slater's condition for strong duality and despite the fact that CMDPs with different initial states in general do \emph{not} share a common Markovian optimal policy.
    \item We instantiate the framework under the linear MDP setting as a proof of concept, creating the first provably correct reset-free RL algorithm to our knowledge that achieves sublinear regret and resets.
\end{enumerate}

%% file: problemformulation.tex
\vspace{-5mm}
\section{Preliminary} \label{sec:prelim}
We consider episodic reset-free RL: in each episode, the agent aims to optimize for a fixed-horizon return starting from the last state of the previous episode or some state that the agent was reset to if reset occurred. 
\vspace{-2mm}
\paragraph{Problem Setup and Notation}
Formally, we can define episodic reset-free RL as a Markov decision process (MDP), $(\mathcal{S}, \mathcal{A}, P, r, H)$, where $\mathcal{S}$ is the state space, $\mathcal{A}$ is the action space, $P = \{P_h\}_{h=1}^H$ is the transition dynamics, $r= \{r_h\}_{h=1}^H$ is the reward function, and $H$ is the task horizon. We assume $P$ and $r$ are unknown. 
We allow $\mathcal{S}$ to be large or continuous but assume $A$ is relatively small so that $\max_{a\in \mathcal{A}}$ can be performed.
We designate the set of reset states, or any states that human intervention normally would have been required, as $\Sreset \subseteq \mathcal{S}$; we do not assume that the agent has knowledge of $\Sreset$. We also do not assume that there is a reset-free action at each state, as opposed to \citet{amani21}. Therefore, the agent needs to plan for the long-term to avoid resets. We assume $r_h: \mathcal{S} \times \mathcal{A} \rightarrow [0,1]$, and for simplicity, we assume $r_h$ is deterministic. However, we note that it would be easy to extend this to the setting where rewards are stochastic.

The agent interacts with the environment for $K$ total episodes.  Following the convention of episodic problems, we suppose the state space $\mathcal{S}$ is layered, and a state $s_t \in \mathcal{S}$ at time $t$ is factored into two components $s_t =(\bar{s}, t)$ where $\bar{s}$ denotes the time-invariant part.
Reset happens at some time $t$ if the time-invariant part of $s_t$, $\bar{s} \in \Sreset$. The initial state of the next episode will be $s_1 = (\bar{s}', 1)$ where $\bar{s}'$ is sampled from an unknown state distribution. In a given episode, if reset happens, we designate the time step this occurs as $t = \tau$.
If there is no reset in an episode, the initial state of the next episode is the last state of the current episode, i.e., for episode $k+1$, $s_1^{k+1} = (\bar{s}, 1)$ if $s_H^k = (\bar{s}, H)$ in episode $k$.\footnote{This setup covers reset-free multi-task or lifelong RL problems that are modeled as contextual MDPs. We can treat each state as $s_\tau = (\bar{s}, c, \tau)$, where $c$ denotes the context that stays constant within an episode. If no reset happens, the initial state of episode $k+1$ is $s_1^{k+1} = (\bar{s}, c^{k+1},  1)$ if $s_H^k = (\bar{s}, c^k, H)$ in episode $k$, where the new context $c^{k+1}$ may depend on the current context $c^k$.} Therefore, the initial state sequence is adaptive. This sequence is necessary to consider since in reset-free RL, we want to avoid resetting, including after each episode.

We denote the set of Markovian policies as $\Delta$, and a policy $\pi\in\Delta$ as $\pi = \{\pi_{h}(a_h| s_h)\}_{h=1}^{H}$.  We define the state value function and the state-action value function under $\pi$ as
\begin{align} \label{eq:value function}
    V_{r,h}^\pi(s) &:= \mathbb{E}_\pi \Big{[}\textstyle \sum_{t=h}^{\min(H, \tau)} r_t(s_t, a_t)|s_h = s\Big{]}\\
    Q_{r,h}^\pi(s,a) &:= r_h(s,a) + \mathbb{E}\Big{[}  V_{r,h+1}^\pi(s_{h+1}) | s_h = s, a_h = a\Big{]}, \nonumber
\end{align}
where  $h\leq \tau$, and we recall $\tau$ is the time step when the agent enters $\Sreset$ (if at all).

\paragraph{Objective}

The overall goal is for the agent to learn a Markovian policy to maximize its cumulative reward while avoiding resets. Therefore, our performance measures are as follows (we seek to minimize both quantities): 
\begin{align} \label{eq:regret and resets}
    &\text{Regret}(K) = \max_{\pi \in \Dfree }\textstyle \sum_{k=1}^K V^{\pi}_{r,1}(s_1^k) - V^{\pi^k}_{r,1}(s_1^k)\\ 
    &\text{Resets}(K) \hspace{-0.5mm} = \hspace{-0.5mm} \textstyle \sum_{k=1}^K \hspace{-0.5mm} \mathbb{E}_{\pi^k} \hspace{-1mm}\left[\textstyle \sum_{h=1}^{\min(H,\tau)} \mathds{1}[s_h \in \Sreset] \Big| s_1 \hspace{-1mm} = \hspace{-1mm} s_1^k\right] 
\end{align}
where $\Dfree\subseteq\Delta$ denotes the set of Markovian policies that avoid resets for all episodes, and $\pi^k$ is the policy used by the agent in episode $k$. Note that by the reset mechanism $\sum_{h=1}^{\min(H,\tau)} \mathds{1}[s_h \in \Sreset] \in \{ 0, 1\} $. 

Notice that the initial states in our regret and reset measures are determined by the learner. 
Given the motivation behind reset-free RL (see \cref{sec:intro}), we can expect that the initial states here are meaningful for performance comparison by construction. Otherwise, a reset would have occurred to set the learner to a meaningful state. 
Note that this means by the reset mechanism, all ``bad" absorbing states are in $\Sreset$, and hence, the agent cannot hide in a ``bad" absorbing state to achieve small regret. 
In addition, since the learner does not receive any reward within an episode after being reset, the learner cannot leverage resets to gain an advantage over the optimal policy (which never resets) to minimize or even achieve negative regret.

%

To make the problem feasible, we assume achieving no resets is possible. We state this assumption formally below.
\begin{assumption} \label{as:reset-free is feasible.}
For any sequence $\{ s_1^k \}_{k=1}^K$, the set $\Dfree$ is not empty. That is, there is a Markovian policy $\pi\in\Delta$ such that 
    $ \mathbb{E}_{\pi} [\sum_{h=1}^H \mathds{1}[s_h \in \Sreset] | s_1 = s_1^k ] = 0 $.
\end{assumption}
This assumption is simply stating that every episode of learning admits a reset-free policy.
As discussed in \cref{sec:related}, this assumption  is weaker than existing assumptions in the  literature. We note that an alternate assumption that only $s_1^1$ (i.e., the initial state of the first episode) admits a reset-free policy is insufficient to make reset-free RL feasible; since the transition dynamics are unknown to the agent, under this assumption alone, for any algorithm, there is a problem\footnote{Consider a state space $\mathcal{S}$ which can be separated into a reset-free state $s^\dagger$, a reset state $\bar{s}$, and a subset $\tilde{\mathcal{S}} = \mathcal{S} \setminus s^\dagger, \bar{s}$. The subset $\tilde{\mathcal{S}}$ contains reset states and reset-free states. For all actions taken at $s^\dagger$, the agent will land at the reset state $\bar{s}$. Given a learning algorithm, let $\tilde{\mathcal{A}}$ be a subset of actions that the learning algorithm has constant probability of taking at $s_1^1$ (the very first initial state). Since the learning agent has no knowledge of the MDP, without violating the assumption that $s_1^1$ admits a reset-free policy, we can construct an MDP such that taking actions in $\tilde{\mathcal{A}}$ would lead to a reset state and a reset mechanism which resets the agent to the reset-free state $s^\dagger$ whenever the agent enters a reset state. As a result, the learning agent has a constant probability of incurring a linear number of resets over the total number of episodes.
} such that the number of resets must be linear. 
%

We highlight that \cref{as:reset-free is feasible.} is a reasonable assumption in practice. It does not require a fixed initial state. In addition,
if reset happens, in practice, the agent is usually set to a state where it can continue to operate in without reset; if the agent is at a state where no such reset-free policy exists, reset should happen. This assumption is similar to the assumption on the existence of a perfectly safe policy in safe RL literature, which is a common and relatively weak assumption~\citep{ghosh22, ding20}. If there were to be initial states that inevitably lead to a reset, the problem would be infeasible and does not follow from the motivation of reset-free RL.
\vspace{-3mm}
\section{A No-Regret Reduction for Reset-Free RL} \label{sec:reduction}

In this section, we present our main reduction of reset-free RL to regret minimization in a two-player game. In the following, we first show that reset-free RL can be framed as a sequence of CMDPs  with an adaptive initial state sequence. 
Then we design a two-player game based on a primal-dual analysis of this sequence of CMDPs. Finally, we show achieving sublinear regret in this two-player game implies sublinear regret and resets in the original reset-free RL problem in \eqref{eq:regret and resets}, and we discuss how to leverage this framework to systematically design reset-free RL algorithms.

Our reduction differs from standard reductions involving bounding the Nash gap with the sum of two players' regret in the literature of constrained optimization and online learning. 
First, we show a reduction for a \emph{sequence} of saddle-point problems instead of for a single fixed saddle-point problem. There are CMDP methods (e.g. \cite{ghosh22} that implicitly use the two players’ regret to bound the Nash gap of a CMDP in their analysis. However, those proofs are applicable to only a single CMDP with a fixed initial state distribution. And importantly, they fundamentally rely on Slater's condition assumption (i.e., requiring a strictly feasible policy), which does not hold for the CMDPs considered here. 
Additionally, unlike the typical bound for a convex-concave saddle-point problem in the optimization literature, \cite{wang20, kov22, boyd14} the saddle-point problem of a CMDP is non-concave in terms of the policy (and convex in terms of the dual function $\lambda$). In this paper, we take a different analysis to bypass the interplaying complexities due to CMDP sequences, non-concavity, and the lack of Slater’s condition. The complete proofs for this section are in \cref{sec:missing proof of reduction}.

\subsection{Reset-free RL as a Sequence of CMDPs}
The first step of our reduction is to cast the reset-free RL problem in \cref{sec:prelim} to a sequence of CMDP problems which share the same rewards, constraints, and dynamics, but have different initial states. Each problem instance in this sequence corresponds to an episode of the reset-free RL problem, and its constraint describes the probability of the agent entering a state that requires reset.  

Specifically, we denote these constrained MDPs
as $\{(\mathcal{S}, \mathcal{A}, P,  r, H, c, s_1^k)\}_{k=1}^K$: in episode $k$, the CMDP problem is defined as 
\begin{equation} \label{eq:mainO}
    \max_{\pi \in \Delta}  V^\pi_{r,1}(s_1^k), \text{ s.t. } V^\pi_{c,1}(s_1^k) \leq 0
\end{equation}
where we define the cost as 
\begin{align*}
    c_h(s, a) \coloneqq \mathds{1}[s \in \Sreset]
\end{align*}
and $ V^\pi_{c,1}$, defined similarly to \eqref{eq:value function}, is the state value function with respect to the cost $c$ . We note that the initial state, $s_1^k$,  depends on the past behaviors of the agent, and that \cref{as:reset-free is feasible.} ensures each CMDP in \eqref{eq:mainO} is a feasible problem (i.e., there is a Markovian policy satisfying the constraint). The objective of the agent in each CMDP is to be reward-maximizing while adhering to the constraint, namely that the number of resets should be $0$. Requiring zero constraint violation instead of a small constraint violation will be crucial for our reduction. 

Since CMDPs are typically defined without early episode termination unlike \eqref{eq:value function}, with abuse of notation, we extend the definitions of $P$, $\mathcal{S}$, $r$, $c$ as follows so that the CMDP definition above is consistent with the common literature. We introduce a fictitious absorbing state denoted as $s^\dagger$ in $\mathcal{S}$, where $r_h(s^\dagger,a)=0$ and $c_h(s^\dagger,a)=0$; once the agent enters $s^\dagger$, it stays there until the end of the episode. We extend the definition $P$ such that, after the agent is in a state $s\in \Sreset$, any action it takes brings it to $s^\dagger$ in the next time step. In this way, we can write the value function, e.g. for reward, as $V_{r,h}^\pi(s) = \mathbb{E}_\pi \Big{[}\sum_{t=h}^H r_t(s_t, a_t)|s_h = s\Big{]}$ in terms of this extended dynamics. We note that these two formulations are mathematically the same for the purpose of learning; when the agent enters $s^\dagger$, it means that the agent is reset in the episode.

By the construction  above, we can write 
\begin{align*}
    \text{Resets}(K) = \textstyle \sum_{k=1}^K V_{c,1}^{\pi^k}(s_1^k)
\end{align*}
which is the same as the number of total constraint violations across the CMDPs. 
Because we do not make any assumptions about the agent's knowledge of the constraint function (e.g., the agent does not know states $\in \Sreset$), we allow the agent to reset during learning while minimizing the total number of resets over all $K$ episodes.

\subsection{Reduction to Two-Player Game}
From the problem formulation above, we see that the major difficulty of reset-free RL is the coupling between an adaptive initial state sequence and the constraint on reset probability. If we were to remove either of them, we could use standard algorithms, since the problem would become a single CMDP problem~\citep{altman1999constrained} or an episodic RL problem with varying initial states \citep{jin19}. 

We propose a reduction to systematically design algorithms for this sequence of CMDPs and therefore for reset-free RL.
The main idea is to approximately solve the saddle-point problem of each CMDP in \eqref{eq:mainO}, i.e.,
\begin{equation}\label{eq:saddle point problem}
    \max_{\pi\in\Delta} \min_{\lambda \geq 0 }  V^\pi_{r,1}(s_1^k) - \lambda V^\pi_{c,1}(s^k_1)
\end{equation}
where $\lambda$ denotes the dual variable (i.e., the Lagrange multiplier). Each CMDP can be framed as a linear program \citep{hern02} whose primal and dual optimal values match (see section $8.1$ in \citet{hazan2016introduction}). Therefore, for each CMDP, $ \max_{\pi\in\Delta} \min_{\lambda \geq 0 }  V^\pi_{r,1}(s_1^k) -\lambda V^\pi_{c,1}(s^k_1) =  \min_{\lambda \geq 0 }  \max_{\pi\in\Delta} V^\pi_{r,1}(s_1^k) - \lambda V^\pi_{c,1}(s^k_1)$. 

While using a primal-dual algorithm to solve for the saddle-point of a \emph{single} CMDP is known, using this approach for a sequence of CMDPs is not obvious. 
This is because in general, each CMDP's optimal policy (and Lagrange multiplier) can be a function of its initial state distribution~\citep{altman1999constrained}.
An easy way to see this is that an optimal policy satisfying a constraint on an initial state $s_1$ may be infeasible for a constraint on another initial state $s_1'$, if, e.g, running this policy starting from $s_1$ does not reach $s_1'$ and therefore can have arbitrary behavior at $s_1'$. This behavior is different from that of unconstrained MDPs, where there is an optimal Markovian policy across all the initial states.


Therefore, in general, a common saddle-point of \emph{Markovian} polices and Lagrange multipliers does not necessarily exist for a sequence of CMDPs with varying initial states.\footnote{A shared saddle-point with a \emph{non-Markovian} policy always exists on the other hand.} As a result, it is unclear if there exists a primal-dual algorithm to solve this sequence, especially given that the initial states here are adaptively chosen.

\paragraph{Existence of a Shared Saddle-Point}
Fortunately, we show that there is a shared saddle-point with a Markovian policy across all the CMDPs considered in \eqref{eq:mainO} due to the special structure of reset-free RL. It is a proof that does not use Slater’s condition for strong duality, unlike similar literature, but attains the desired property. Instead, it follows from the fact that we impose a constraint that requires zero constraint violations. In addition, we also use \cref{as:reset-free is feasible.} and the fact that the cost c is non-negative. We formalize this below.

\begin{restatable}{theorem}{saddlePt}
\label{th:existence of shared saddle-point}
There\hspace{-0.1mm} exists\hspace{-0.1mm} a \hspace{-0.1mm}function\hspace{-0.1mm} $\hat{\lambda}(\cdot)$\hspace{-0.1mm} where \hspace{-0.1mm}for\hspace{-0.1mm} each \hspace{-0.2mm} $s$\hspace{-0.2mm},
    \begin{align*}
        \hat{\lambda}(s) \in \arg\min_{y\geq 0} \left( \max_{\pi \in \Delta}  V^\pi_{r,1}(s) - y V^\pi_{c,1}(s) \right),
    \end{align*}
    and a Markovian policy $\pi^*\in\Delta$, such that 
    $(\pi^*, \hat{\lambda})$ is a saddle-point to the CMDPs
    \[
    \max_{\pi \in \Delta}  V^\pi_{r,1}(s_1), \text{ s.t. } V^\pi_{c,1}(s_1) \leq 0
    \]
    for \emph{all} initial states $s_1\in\mathcal{S}$ such that the CMDP is feasible. That is, for all $\pi\in\Delta$, $\lambda:\mathcal{S}\to\mathbb{R}$, and $s_1\in\mathcal{S}$,
    \begin{align} \label{eq:saddle-point}
    \hspace{-3mm}
          V^{\pi^*}_{r,1}(s_1) - \lambda(s_1) V^{\pi^*}_{c,1}(s_1) 
         &\geq  V^{\pi^*}_{r,1}(s_1) - \hat{\lambda}(s_1) V^{\pi^*}_{c,1}(s_1) \nonumber \\
          &\geq  V^{\pi}_{r,1}(s_1) - \hat{\lambda}(s_1) V^{\pi}_{c,1}(s_1).
    \end{align}
\end{restatable}
\begin{restatable}{corollary}{saddlePtCor}
\label{lm:regret equivalence}
For $\pi^*$ in \cref{th:existence of shared saddle-point}, it holds that $\text{Regret}(K) = \sum_{k=1}^K V_{r,1}^{\pi^*}(s_1^k) - V_{r,1}^{\pi^k}(s_1^k)$. 
\end{restatable}

We prove for ease of construction that the pair $(\pi^*, \lambda^*)$ where $\lambda^*(\cdot) = \hat{\lambda}(\cdot)+1$ is also a saddle-point. 
\begin{restatable}{corollary}{lambdaPlus}
\label{lemma:additional saddle pt}
For any saddle-point to the CMDPs 
\[
    \max_{\pi \in \Delta}  V^\pi_{r,1}(s_1), \text{ s.t. } V^\pi_{c,1}(s_1) \leq 0
\]
of $(\pi^*, \hat{\lambda})$ from \cref{th:existence of shared saddle-point}, $(\pi^*, \lambda^*) \eqqcolon (\pi^*, \hat{\lambda}+1)$ is also a saddle-point as defined in \eqref{eq:saddle-point}.
\end{restatable}

Therefore, the pair $(\pi^*, \lambda^*)$ in \cref{lemma:additional saddle pt} is a saddle-point to all the CMDPs the agent faces. This makes potentially designing a two-player game reduction possible. Now we give the details of our construction.

\vspace{-2mm}
\paragraph{Two-Player Game}
Our two-player game proceeds iteratively in the following manner: 
in episode $k$, a dual player determines a state value function $\lambda^k:\mathcal{S}\to\mathbb{R}$, and then, a primal player determines a policy $\pi^k$ which can depend on $\lambda^k$. The primal and dual players then receive losses $\Lagr^k(\pi^k, \lambda)$ and $-\Lagr^k(\pi, \lambda^k)$, respectively, where $\Lagr^k(\pi, \lambda)$ is a Lagrangian function defined as 
\begin{equation}\label{eq:lagrangian}
    \Lagr^k(\pi, \lambda) := V^\pi_{r,1}(s_1^k) - \lambda(s_1^k) V^\pi_{c,1}(s^k_1).
\end{equation}
The regret of these two players are defined as follows.
\begin{definition} \label{regretXY}
Let $\pi_c$ and $\lambda_c$ be comparators. The regret of the primal and the dual players are
\begin{align*}
    R_p(\{\pi^k\}_{k=1}^K, \pi_c) \coloneqq \textstyle \sum_{k=1}^K \Lagr^k(\pi_c, \lambda^k) - \Lagr^k(\pi^{k}, \lambda^k), \\
    R_d(\{\lambda^k\}_{k=1}^K, \lambda_c) \coloneqq  \textstyle \sum_{k=1}^K \Lagr^k(\pi^k, \lambda^k) - \Lagr^k(\pi^k, \lambda_c).
\end{align*}
\end{definition}
We present our main reduction theorem  below. 

\begin{restatable}{theorem}{regretR}
\label{th:regret reduction}
Under \cref{as:reset-free is feasible.}, for any sequences $\{\pi^k\}_{k=1}^K$ and $\{\lambda^k\}_{k=1}^K$ , it holds that
\begin{align*}
    \text{Regret}(K) &\leq R_p(\{\pi^k\}_{k=1}^K, \pi^*) + R_d(\{\lambda^k\}_{k=1}^K, 0)\\
\text{Resets}(K) &\leq R_p(\{\pi^k\}_{k=1}^K, \pi^*) +R_d(\{\lambda^k\}_{k=1}^K, \lambda^*)
\end{align*}
where $(\pi^*,\lambda^*)$ is the saddle-point defined in \cref{lemma:additional saddle pt}.
\end{restatable}

\paragraph{Designing Reset-free RL Algorithms by Reduction}
By \cref{th:regret reduction}, if both players have sublinear regret in the two-player game, then the resulting policy sequence will have sublinear performance regret and a sublinear number of resets in the original RL problem. 
Minimization of regret in a two-player game and solving saddle-point problems have been very well studied (see, e.g., \citep{nguyen18,benzi05,hazan2016introduction}). Therefore, our reduction gives a constructive and provably correct template for designing reset-free RL algorithms. 
It should be noted however, that the resultant two-player game here is slightly different from the classical full-information online learning problems. Unlike in these problems, the learning agent needs to actively explore the environment in each episode to collect information about the unknown dynamics, reward, and cost (i.e., which states are reset-free). Specifically \cref{th:regret reduction} needs the primal (policy) player to solve an online MDP sequence (with the same transition dynamics but varying reward functions due to changing $\lambda$), and the dual player to solve an online linear problem.
In the next section, we will give an example algorithm under this reduction for linear MDPs.
\subsection{Proof Sketches}

\paragraph{Proof Sketch of \cref{th:existence of shared saddle-point}}

Let $Q_c^*(s,a) = \min_{\pi\in\Delta} Q_c^\pi(s,a)$ and $V_c^*(s) = \min_{\pi\in\Delta} V_c^\pi(s)$. 
We define $\pi^*$ in \cref{th:existence of shared saddle-point} as the optimal policy to the following MDP: $(\mathcal{S}, \overline{\mathcal{A}}, P, r, H)$, where we define a state-dependent action space $\overline{\mathcal{A}}$ as 
\begin{align*}
    \overline{\mathcal{A}}_s = \{ a \in \mathcal{A} :  Q_c^*(s,a) \leq V_c^*(s) \}.
\end{align*}
By definition, $\overline{\mathcal{A}}_s $ is non-empty for all $s$. 

We also define a shorthand notation: we write $ \pi \in \overline{\mathcal{A}}(s)$ if 
$\mathbb{E}_{\pi}[ \sum_{t=1}^H \mathds{1}\{ a_t \notin   \overline{\mathcal{A}}_{s_t} \} | s_1 =s ] = 0$. We have the following lemma, which is an application of the performance difference lemma (see Lemma $6.1$ in~\citep{kakade2002approximately} and Lemma A.1 in \citep{cheng2021heuristic}).
\begin{restatable}{lemma}{restrPolicy}
\label{lm:optimality of restricted policies main}
For any $s_1 \in \mathcal{S}$ such that $V_c^* (s_1)=0$ and any $\pi\in\Delta$, it is true that $ \pi \in \overline{\mathcal{A}}(s_1)$ if and only if $V_c^\pi(s_1)=0$.
\end{restatable}

We prove our main claim, \eqref{eq:saddle-point}, below. 
Because $V^{\pi^*}_{c,1}(s_1)  = 0$, the first inequality is trivial: 
$
    V^{\pi^*}_{r,1}(s_1) - \lambda(s_1) V^{\pi^*}_{c,1}(s_1)  
    = V^{\pi^*}_{r,1}(s_1)
    = V^{\pi^*}_{r,1}(s_1) - \hat{\lambda}(s_1) V^{\pi^*}_{c,1}(s_1) 
$.

To prove the second inequality, we use \cref{lm:optimality of restricted policies main}: 
\begin{align*}
     &V^{\pi}_{r,1}(s_1) - \hat{\lambda}(s_1) V^{\pi}_{c,1}(s_1) \\
     \leq 
     &\max_{\pi\in\Delta} V^{\pi}_{r,1}(s_1) - \hat{\lambda}(s_1) V^{\pi}_{c,1}(s_1) \\
     = &\min_{y\geq 0} \max_{\pi \in \Delta}  V^\pi_{r,1}(s_1) - y V^\pi_{c,1}(s_1)     \\
    =  &\max_{\pi \in \overline{\mathcal{A}_c}(s_1)}  V^\pi_{r,1}(s_1) & \textrm{(By  \cref{lm:optimality of restricted policies main} )}\\
    = &V^{\pi^*}_{r,1}(s_1) = V^{\pi^*}_{r,1}(s_1) - \hat{\lambda}(s_1) V^{\pi^*}_{c,1}(s_1). 
\end{align*}

\paragraph{Proof Sketch of \cref{th:regret reduction}}

We first establish the following intermediate result that will help us with our decomposition. 
\begin{restatable}{lemma}{rdiffLemma}
\label{diffLemma}
For any primal-dual sequence $\{\pi^k, \lambda^k \}_{k=1}^K$, $\sum_{k=1}^K (\Lagr^k(\pi^*, \lambda') - \Lagr^k(\pi^k, \lambda^k)) \leq R_p(\{\pi\}_{k=1}^K, \pi^*)$, where $(\pi^*,\lambda')$ is the saddle-point defined in either \cref{th:existence of shared saddle-point} or \cref{lemma:additional saddle pt}.
\end{restatable}
Then we upper bound Regret($K$) and Resets($K$) by $R_p(\{\pi^k\}_{k=1}^K, \pi_c)$ and $R_d(\{\lambda^k\}_{k=1}^K, \lambda_c)$ for suitable comparators. This decomposition is inspired by the techniques used in \citet{nguyen18}. We first bound Resets($K$).
\begin{restatable}{lemma}{rCV}
\label{lm:constraint violation}
For any primal-dual sequence $\{\pi^k, \lambda^k \}_{k=1}^K$, $\sum_{k=1}^K V_{c,1}^{\pi^k}(s^k_1) \leq R_p(\{\pi\}_{k=1}^K, \pi^*) + R_d(\{\lambda\}_{k=1}^K, \lambda^*) $, where $(\pi^*,\lambda^*)$ is the saddle-point defined in \cref{lemma:additional saddle pt}.
\end{restatable}
\begin{proof}
Notice $\sum_{k=1}^K V_{c,1}^{\pi^k}(s^k_1) = \sum_{k=1}^K \Lagr^k(\pi^k, \hat{\lambda}) - \Lagr^k (\pi^k, \lambda^{*})$ where $(\pi^*,\hat{\lambda})$ is the saddle-point defined in \cref{th:existence of shared saddle-point}. By \eqref{eq:saddle-point}, and adding and subtracting $\sum_{k=1}^K \Lagr^k(\pi^k, \lambda^k)$, we can bound this difference by
\[
\textstyle \sum_{k=1}^K \Lagr^k(\pi^*, \hat{\lambda}) - \Lagr^k(\pi^k, \lambda^k) + \Lagr^k(\pi^k, \lambda^k) - \Lagr^k(\pi^k, \lambda^{*}).
\]
Using~\Cref{diffLemma} and \cref{regretXY} to upper bound the above, we get the desired result.
\end{proof}
Lastly, we bound Regret$(K)$ with the lemma below and \cref{lm:regret equivalence}.
\begin{restatable}{lemma}{rregret}
\label{decompR}
For any primal-dual sequence $\{\pi^k, \lambda^k \}_{k=1}^K$, $\sum_{k=1}^K (V_{r,1}^{\pi^*}(s_1^k) - V_{r,1}^{\pi^k}(s_1^k)) \leq R_p(\{\pi\}_{k=1}^K, \pi^*) + R_d(\{\lambda\}_{k=1}^K, 0)$, where $(\pi^*,\lambda^*)$ is the saddle-point defined in \cref{lemma:additional saddle pt}.
\end{restatable}
\begin{proof}
Note that $\Lagr(\pi^*, \lambda^*) = \Lagr(\pi^*, 0)$ since $V_{c,1}^{\pi^*} = 0$ for all $k \in [K] = \{1,...,K \}$. Since by definition, for any $\pi$, $\Lagr^k(\pi, 0) = V_{r,1}^{\pi}(s_1^k)$, we have the following: 
\begin{align*}
&\textstyle \hspace{-4mm} \sum_{k=1}^K \hspace{-0.5mm} V_{r,1}^{\pi^*}(s_1^k) - V_{r,1}^{\pi^k}(s_1^k)\hspace{-0.5mm}= \hspace{-0.5mm}\sum_{k=1}^K \Lagr^k(\pi^*, \lambda^*) - \Lagr^k(\pi^k, 0) \\
    = &\textstyle \sum_{k=1}^K \Lagr^k(\pi^*, \lambda^*) \hspace{-0.5mm}-\hspace{-0.5mm} \Lagr^k(\pi^k, \lambda^k) \hspace{-0.5mm}+\hspace{-0.5mm} \Lagr^k(\pi^k, \lambda^k) \hspace{-0.5mm}-\hspace{-0.5mm} \Lagr^k(\pi^k, 0)\\
     \leq &R_p(\{\pi\}_{k=1}^K, \pi^*) + R_d(\{\lambda\}_{k=1}^K, 0)
\end{align*}
where the last inequality follows from~\Cref{diffLemma} and \cref{regretXY}. 
\end{proof}

%% file: algorithm.tex
\vspace{-5mm}
\section{Reset-Free Learning for Linear MDP} \label{sec:linear MDP}

To demonstrate the utility of our reduction, we design a provably correct reset-free algorithm instantiation for linear MDP. This example serves to ground our abstract framework and to illustrate concretely how an algorithm instantiating our framework might operate. We also aim to show that despite generality, our framework does not lose any optimality in the rates of regret and total number of resets when compared to similar but specialized works.


We consider a linear MDP setting, which is common in the RL theory literature~\citep{jin19, amani21, ghosh22}.
\begin{assumption}\label{as:linear mdp}
We assume $(S, A, P, r, c, H)$ is linear with a known feature map $\phi: S \times A \rightarrow \mathbb{R}^d$:  for any $h \in [H]$, there exists $d$ unknown signed measures $\mu_h = \{ \mu_h^1, ..., \mu_h^d\}$ over $S$ such that for any $(s,a, s') \in S \times A \times S$, we have
\begin{align*}
    P_h(s'|a) &= \langle \phi(s, a), \mu_h(s') \rangle,
\end{align*}
and there exists unknown vectors $\omega_{r,h}, \omega_{c,h} \in \mathbb{R}^d$ such that for any $(s,a) \in S \times A$,
\begin{align*}
    r_h(s,a) &= \langle \phi(s,a), \omega_{r,h} \rangle, 
    \hspace{3mm} c_h(s,a) = \langle \phi(s,a), \omega_{c,h} \rangle.
\end{align*}
We assume,  for all $(s, a, h) \in S \times A \times [H]$,  $||\phi(s,a)||_2 \leq 1$, and $\max\{||\mu_h(s)||_2, ||\omega_{r,h}||_2, ||\omega_{c,h}||_2\} \leq \sqrt{d}$.
\end{assumption}
Note that the above assumption on the cost function does not imply knowledge of the reset states nor some hidden structure of the reset states. At the high level, it merely asks that the feature is expressive enough to separate the reset states and reset-free states in a classification problem.

In addition, we make a  linearity assumption on the function $\lambda^*$ defined in \cref{th:existence of shared saddle-point}.
\begin{assumption} \label{as:bounded dual variable}
We assume the knowledge of a feature $\xi: \mathcal{S} \rightarrow \mathbb{R}^d$ such that $\forall s\in \mathcal{S}$,
 $||\xi(s)||_2 \leq 1$ and $\lambda^*(s)= \langle \xi(s), \theta^* \rangle$ for some unknown vector $\theta^* \in \mathbb{R}^d$. In addition, we assume the knowledge of a convex set\footnote{Such a set can be constructed by upper bounding the values using scaling and ensuring non-negativity by, e.g., sum of squares.} $\mathcal{U}\subseteq \mathbb{R}^d$ such that  $\theta^*, 0 \in \mathcal{U}$  
 and $\forall 
 \theta \in \mathcal{U}$, $||\theta||_2 \leq B$ and $\langle \xi(s), \theta\rangle \geq 0$.
 \footnote{
 From the previous section, we can see that the optimal function for the dual player is not necessarily unique. So, we assume bounds on at least one optimal function that we designate as $\lambda^*(s)$.
 }
\end{assumption}
 For a linear CMDP with a fixed initial state distribution, it is standard to assume that the optimal Lagrange multiplier (i.e., $\lambda^*(s_1)$) is bounded and non-negative. In this regard, our assumption is a mild and natural linear extension of this boundedness assumption for varying initial states, and is true when the feature map is expressive enough.
\vspace{-1mm}
\subsection{Algorithm}
\begin{algorithm*}[t]
\caption{Primal-Dual Reset-Free RL Algorithm for Linear MDP with Adaptive Initial States} \label{mainalg}
\begin{algorithmic}[1]
\STATE \textbf{Input}: Feature maps $\phi$ and $\xi$. Failure probability $p$. Some universal constant $\mathcal{C}$.
\STATE \textbf{Initialization:} $\theta^1 = 0$, $w_{r,h} = 0$, $w_{c,h} = 0$, $\alpha = \dfrac{\log(|\mathcal{A}|)K}{2(1+B+H)}$, $\beta = \mathcal{C}dH\sqrt{\log(4\log|\mathcal{A}|dKH/p)}$
\FOR {episodes $k = 1,...K$}
    \STATE Observe the initial state $s_1^k$.
    \FOR {step $h=1,...,H$}\label{ln:policy}
        \STATE Compute $\pi_{h,k}(a|\cdot) \gets \dfrac{\exp(\alpha (Q^k_{r,h}(\cdot, a)-\lambda^k(s_1^k)Q^k_{c,h}(\cdot ,a)))}{\sum_a \exp(\alpha (Q^k_{r,h}(\cdot, a) -\lambda^k(s_1^k)Q^k_{c,h}(\cdot, a)))}$.
        \STATE Take action $a^k_h \sim \pi_{h,k}(\cdot | s^k_h)$ and observe $s^k_{h+1}$.
    \ENDFOR \label{end:policy}
    \STATE $\eta_k \gets B/\sqrt{k}$ \label{ln:Y}
    \STATE Update $\theta^{k+1} \gets \text{Proj}_{\mathcal{U}}(\theta^{k} + \eta_k \cdot \xi(s^{k}_1) V^{k}_{c,1}(s^{k}_1) )$
    \STATE $\lambda^{k+1}(\cdot) \gets  \langle \theta^{k+1}, \xi(\cdot) \rangle$ \label{end:Y}
    \FOR{step $h = H,...,1$} \label{ln:primal}
        \STATE $\Lambda_h^{k+1}\gets \sum\limits_{i = 1}^{k} \phi(s^i_h, a^i_h)\phi(s^i_h, a^i_h)^T + \lambda \mathbb{I}$.
        \STATE $w^{k+1}_{r,h} \gets (\Lambda_h^{k+1})^{-1}[\sum\limits_{i = 1}^{k} \phi(s^i_h, a^i_h)[r_h(s^i_h, a^i_h) + V^{k+1}_{r,h+1}(s^i_{h+1})]]$
        \STATE $w^{k+1}_{c,h} \gets (\Lambda_h^{k+1})^{-1}[\sum\limits_{i = 1}^{k} \phi(s^i_h, a^i_h)[c_h(s^i_h, a^i_h) + V^{k+1}_{c,h+1}(s^i_{h+1})]]$
        \STATE $Q^{k+1}_{r,h}(\cdot, \cdot)\gets \max\{ \min \{\langle w^{k+1}_{r,h}, \phi(\cdot, \cdot) \rangle + \beta (\phi(\cdot, \cdot)^T(\Lambda^{k+1}_h)^{-1}\phi(\cdot, \cdot))^{1/2},H-h+1\}, 0\}$
        \STATE $Q^{k+1}_{c,h}(\cdot, \cdot) \gets \max\{ \min\{\langle w^{k+1}_{c,h}, \phi(\cdot, \cdot) \rangle- \beta (\phi(\cdot, \cdot)^T(\Lambda^{k+1}_h)^{-1}\phi(\cdot, \cdot))^{1/2}, 1 \}, 0 \}$ 
        \STATE $V^{k+1}_{r,h}(\cdot) = \sum_a \pi_{h,k}(a|\cdot)Q^{k+1}_{r,h}(\cdot, a)$
        \STATE $V^{k+1}_{c,h}(\cdot) = \sum_a \pi_{h,k}(a|\cdot)Q^{k+1}_{c,h}(\cdot, a)$
    \ENDFOR\label{end:primal}
\ENDFOR
\end{algorithmic}
\end{algorithm*}
The basis of our algorithm lies between the interaction between the primal and dual players.
We adopt the common no-regret plus best-response approach~\citep{wang2021no} to the no-regret two-player game in designing algorithms for these two players.
We let the dual player perform (no-regret) projected gradient descent and the primal player update policies based on upper confidence bound (UCB) with knowledge of the dual player's decision (i.e., the best-response scheme to the dual player). 

As discussed above, the environment is unknown to the agent. Therefore, as discussed in \cref{sec:reduction}, we need to modify the actions of the primal player as compared to in the generic saddle-point problem to instantiate our framework. We do this using UCB, where the agent views unknown actions/states in a more positive (``optimistic")\footnote{Note that the term ``optimism'' here refers to the use of overestimation in exploration, which is different from the usage in, e.g., optimistic mirror descent~\citep{mertikopoulos2018optimistic}.
} way than the ones that have been observed to drive exploration in unknown environments.



Specifically, in each episode, upon receiving the initial state, we execute actions according to the policy based on a softmax 
(lines \hyperref[ln:policy]{5}-\hyperref[end:policy]{8}). Then, we perform the dual update through projected gradient descent. The dual player plays for the next round, $k+1$, after observing its loss after the primal player plays for the current round, $k$. The projection is to a $l_2$ ball containing $\lambda^*(\cdot)$ 
(lines \hyperref[ln:Y]{9}-\hyperref[end:Y]{11}). Finally, we perform the update of the primal player by computing the $Q$-functions for both the reward and cost with a bonus to encourage exploration 
(lines \hyperref[ln:primal]{12}-\hyperref[end:primal]{20}).

This algorithm uses \citet{ghosh22} as the baseline for the primal player. However, we emphasize that our algorithm can handle the adaptive initial state sequence that is seen in reset-free RL \cref{th:existence of shared saddle-point,th:regret reduction}.
\vspace{-3mm}
\subsection{Analysis}
We show below that our algorithm achieves regret and number of resets that are sublinear in the total number of time steps, $KH$, using \cref{th:regret reduction}. This result is asymptotically equivalent to \citet{ghosh22} and comparable to the bounds of $\tilde{O}(\sqrt{d^2H^6K})$ from \citet{ding20}. 
Therefore, we do not make sacrifices in terms of the regret and total number of resets when specializing our abstract framework. 

\begin{restatable}{theorem}{rLinearMDP}
\label{th:linear mdp special case}
Under Assumptions~\ref{as:reset-free is feasible.}, \ref{as:linear mdp}, and \ref{as:bounded dual variable}, with high probability,
$\text{Regret}(K) \leq \tilde{O} ((B+1)\sqrt{d^3H^4K})$ and $\text{Resets}(K) \leq \tilde{O} ((B+1)\sqrt{d^3H^4K})$.
\end{restatable}
\vspace{-3mm}
\paragraph{Proof Sketch of \cref{th:linear mdp special case}}

We provide a proof sketch here and defer the complete proof to \cref{sec:missing proofs of linear mdps}.
We first bound the regret of $\{\pi^k\}_{k=1}^K$ and $\{\lambda^k\}_{k=1}^K$ and then use this to prove the bounds on our algorithm's regret and number of resets with \cref{th:regret reduction}.

We first bound the regret of $\{\lambda^k\}_{k=1}^K$.

\begin{restatable}{lemma}{rregretL}
\label{regretLambda}
Consider $\lambda_c(s) = \langle \xi(s), \theta_c \rangle$ for some $\theta_c \in \mathcal{U}$. Then it holds that 
$
    R_d(\{\lambda^k\}_{k=1}^K, \lambda_c)  \leq 1.5 B \sqrt{K}
    + \sum_{k=1}^K (\lambda^k(s_1^k) - \lambda_c(s_1^k))(V^k_{c,1}(s_1^k) - V^{\pi^k}_{c,1}(s_1^k))
$. 
\end{restatable}
\begin{proof}
We notice first an equality. 
\begin{align*}
    &R_d(\{\lambda^k\}_{k=1}^K, \lambda_c) =\textstyle \sum_{k=1}^K \Lagr^k(\pi^k, \lambda^k) - \Lagr^k(\pi^k, \lambda_c)\\    
    &= \textstyle \sum_{k=1}^K \lambda_c(s_1^k)V^{\pi^k}_{c,1}(s_1^k) - \lambda^k(s_1^k)V^{\pi^k}_{c,1}(s_1^k) \\
    &= \textstyle \sum_{k=1}^K  (\lambda^k(s_1^k) - \lambda_c(s_1^k))(- V^k_{c,1}(s_1^k))  \\
    &+ \textstyle \sum_{k=1}^K   (\lambda^k(s_1^k) - \lambda_c(s_1^k))(V^k_{c,1}(s_1^k) - V^{\pi^k}_{c,1}(s_1^k)). 
\end{align*}

We observe that the first term is an online linear problem for $\theta^k$ (the parameter of $\lambda^k(\cdot)$). In episode $k \in [K]$, $\lambda^k$ is played, and then the loss is revealed. Since the space of $\theta^k$ is convex, we use standard results (see Lemma $3.1$ in \citet{hazan2016introduction}) to show that updating $\theta^k$ through projected gradient descent results in an upper bound for $\sum_{k=1}^K  (\lambda^k(s_1^k) - \lambda_c(s_1^k))(- V^k_{c,1}(s_1^k))$. 
\end{proof}

We now bound the regret of $\{\pi\}_{k=1}^K$.
\begin{restatable}{lemma}{rregretPi}
\label{regretPi}
Consider any $\pi_c$. With high probability, 
$
R_p(\{\pi\}_{k=1}^K, \pi_{c}) \leq 2H(1+B+H) + 
 \sum_{k=1}^K  V^k_{r,1}(s_1^k) - V^{\pi^k}_{r,1} 
 (s_1^k) + \lambda^k(s^k_1)( V^{\pi^k}_{c,1}(s^k_1) - V^k_{c,1}(s^k_1)).
$
\end{restatable}
\vspace{-4mm}
\begin{proof}
First we expand the regret into two terms.
\begin{small}
\begin{align*}
    &R_p(\{\pi\}_{k=1}^K, \pi_{c}) = \textstyle \sum_{k=1}^K \Lagr^k(\pi_{c}, \lambda^k) - \Lagr^k(\pi^{k}, \lambda^k)\\
    = &\textstyle \sum_{k=1}^K \hspace{-0.5mm} V^{\pi_c}_{r,1}(s_1^k) \hspace{-0.5mm}-\hspace{-0.5mm} \lambda^k(s^k_1)V^{\pi_c}_{c,1}(s^k_1)\hspace{-0.5mm}-\hspace{-0.5mm} [V^{\pi^k}_{r,1}(s_1^k) \hspace{-0.5mm}-\hspace{-0.5mm} \lambda^k(s^k_1)V^{\pi^k}_{c,1}(s^k_1)]\\
    = &\textstyle\sum_{k=1}^K \hspace{-0.5mm} V^{\pi_c}_{r,1}(s_1^k) \hspace{-0.5mm}-\hspace{-0.5mm} \lambda^k(s^k_1)V^{\pi_c}_{c,1}(s^k_1) \hspace{-0.5mm}-\hspace{-0.5mm} [V^k_{r,1}(s_1^k) \hspace{-0.5mm}-\hspace{-0.5mm} \lambda^k(s^k_1)V^k_{c,1}(s^k_1)] \\
    &+ \textstyle\sum_{k=1}^K  V^k_{r,1}(s_1^k) \hspace{-0.5mm}-\hspace{-0.5mm} V^{\pi^k}_{r,1} 
 (s_1^k) \hspace{-0.5mm}+\hspace{-0.5mm} \lambda^k(s^k_1)( V^{\pi^k}_{c,1}(s^k_1) \hspace{-0.5mm}-\hspace{-0.5mm} V^k_{c,1}(s^k_1)).
\end{align*}
\end{small}
To bound the first term, we use Lemma $3$ from \citet{ghosh22}, which characterizes the property of upper confidence bound. 
\end{proof}
Lastly, we derive a bound on $R_d(\{\lambda^k\}_{k=1}^K,\lambda_c) + R_p(\{\pi^k\}_{k=1}^K, \pi_c)$, which directly implies our final upper bound on Regret($K$) and Resets($K$) in \cref{th:linear mdp special case} by \cref{th:regret reduction}. Combining the upper bounds in \cref{regretLambda} and \cref{regretPi}, we have a high-probability upper bound of
\begin{align*}
&R_d(\{\lambda^k\}_{k=1}^K,\lambda_c) + R_p(\{\pi^k\}_{k=1}^K, \pi_c) \\
    &\leq 1.5 B \sqrt{K} + 2H(1+B+H) + \\
    &+ \textstyle\sum_{k=1}^K \hspace{-0.5mm} V^k_{r,1}(s_1^k) \hspace{-0.5mm}-\hspace{-0.5mm} V^{\pi^k}_{r,1}(s_1^k) \hspace{-0.5mm}+\hspace{-0.5mm} \lambda_c(s_1^k) ( V^{\pi^k}_{c,1}(s_1^k) \hspace{-0.5mm}-\hspace{-0.5mm} V^k_{c,1}(s_1^k) )
\end{align*}
where the last term is the overestimation error due to optimism. Note that for all $k \in [K]$, $V^k_{r,1}(s_1^k)$ and $V^k_{c,1}(s_1^k)$ are as defined in \cref{mainalg} and are optimistic estimates of $V^{\pi^*}_{r,1}(s_1^k)$ and $V^{\pi^*}_{c,1}(s_1^k)$. To bound this term, we use Lemma $4$ from \citep{ghosh22}.
\vspace{-2mm}
\subsection{Other Possible Instantiations}
We demonstrated above that our general framework can be instantiated to achieve sublinear regret and total number of resets. Importantly, our algorithm serves as an example of how our general framework can be used to systematically design new algorithms for reset-free RL. We can leverage the multitude of existing algorithms that aim to minimize regret in a two-player game. An example of a different strategy is using a no-regret algorithm like optimistic mirror descent~\citep{mertikopoulos2018optimistic} for the dual player. We can also replace UCB for the primal player with an online MDP no-regret algorithm such as a variation of POLITEX~\citep{chad19}. Further studying different combinations of baseline algorithms is an interesting future research direction, which perhaps could even improve existing algorithms for more specialized settings.

%% file: conclusion.tex
\vspace{-3mm}
\section{Conclusion}
We propose a generic no-regret reduction for designing provable reset-free RL algorithms. Our reduction casts reset-free RL into the regret minimization problem of a two-player game, for which many existing no-regret algorithms are available. As a result, we can reuse these techniques, and future better techniques, to systematically build new reset-free RL algorithms. In particular, we design a reset-free RL algorithm for linear MDPs using our new reduction techniques, taking the first step towards designing provable reset-free RL algorithms. Extending these techniques to nonlinear function approximators and verifying their effectiveness empirically are important future research directions.

%% file: appen.tex
\section{Appendix}
\subsection{Missing Proofs for ~\Cref{sec:reduction}} \label{sec:missing proof of reduction}
\subsubsection{Proof of \cref{th:existence of shared saddle-point}}
\saddlePt*

For policy $\pi^*$, we define it by the following construction (we ignore writing out the time dependency for simplicity): 
first, we define a cost-based MDP $\mathcal{M}_c = (\mathcal{S}, \mathcal{A}, P, c, H)$.
Let $Q_c^*(s,a) = \min_{\pi\in\Delta} Q_c^\pi(s,a)$ and $V_c^*(s) = \min_{\pi\in\Delta} V_c^\pi(s)$ be the optimal values, where we recall $V_c^\pi$ and $Q_c^\pi$ are the state and state-action values under policy $\pi$ with respect to the cost.
Now we construct another reward-based MDP  $\overline{\mathcal{M}} = (\mathcal{S}, \overline{\mathcal{A}}, P, r, H)$, where we define the state-dependent action space $\overline{\mathcal{A}}$ as 
\begin{align*}
    \overline{\mathcal{A}}_s = \{ a \in \mathcal{A} :  Q_c^*(s,a) \leq V_c^*(s) \}.
\end{align*}
By definition, $\overline{\mathcal{A}}_s $ is non-empty for all $s$. We define a shorthand notation: we write $ \pi \in \overline{\mathcal{A}}(s)$ if 
$\mathbb{E}_{\pi}[\sum_{t=1}^H \mathds{1}\{ a_t \notin   \overline{\mathcal{A}}_{s_t} \} | s_1 = s  ] = 0$. Then we have the following lemma, which is a straightforward application of the performance difference lemma.
\restrPolicy*
\begin{proof}
By performance difference lemma~\citep{kakade2002approximately}, we can write 
\begin{align*}
    V_c^\pi(s_1) - V_c^* (s_1)
    = \mathbb{E}_{\pi} \left[ \sum_{t=1}^H  Q_c^*(s_t, a_t) - V_c^* (s_t) |s_1 = s_1  \right].
\end{align*}
If for some $s_1 \in \mathcal{S}$, $\pi \in \overline{\mathcal{A}}(s_1)$, then $\mathbb{E}_{\pi} \left[ \sum_{t=1}^H  Q_c^*(s_t, a_t) - V_c^* (s_t) \right] \leq 0$, which implies $V_c^\pi(s_1) \leq V_c^* (s_1)$. But since $V_c^*$ is optimal, $V_c^\pi(s_1)  = V_c^* (s_1)$.
On the other hand, suppose $V_c^\pi(s_1)=0$. It implies $\mathbb{E}_{\pi} \left[ \sum_{t=1}^H  Q_c^*(s_t, a_t) - V_c^* (s_t)   \right] = 0$ since $V^*_c(s_1) = 0$. Because by definition of optimality $ Q_c^*(s_t, a_t) -V_c^* (s_t) \geq 0 $, this implies $\pi \in \overline{\mathcal{A}}(s_1)$.
\end{proof}

We set our candidate policy $\pi^*$ as the optimal policy of this $\overline{\mathcal{M}}$.
By \cref{lm:optimality of restricted policies main}, we have $V_c^{\pi^*}(s) =  V_c^*(s)$, so it is also an optimal policy to $\mathcal{M}_c$. We prove our main claim of \cref{th:existence of shared saddle-point} below:
    \begin{align*}
          V^{\pi^*}_{r,1}(s_1) - \lambda(s_1) V^{\pi^*}_{c,1}(s_1) 
         \geq  V^{\pi^*}_{r,1}(s_1) - \hat{\lambda}(s_1) V^{\pi^*}_{c,1}(s_1) 
          \geq  V^{\pi}_{r,1}(s_1) - \hat{\lambda}(s_1) V^{\pi}_{c,1}(s_1). 
    \end{align*}
\begin{proof}
    Because $V^{\pi^*}_{c,1}(s_1)  = 0$ (for an initial state $s_1$ such that the CMDP is feasible), the first inequality is trivial: 
    \begin{align*}
        V^{\pi^*}_{r,1}(s_1) - \lambda(s_1) V^{\pi^*}_{c,1}(s_1)  
        = V^{\pi^*}_{r,1}(s_1)
        = V^{\pi^*}_{r,1}(s_1) - \hat{\lambda}(s_1) V^{\pi^*}_{c,1}(s_1). 
    \end{align*}
    For the second inequality, we use \cref{lm:optimality of restricted policies main}: 
    \begin{align*}
         V^{\pi}_{r,1}(s_1) - \hat{\lambda}(s_1) V^{\pi}_{c,1}(s_1) 
         &\leq 
         \max_{\pi\in\Delta} V^{\pi}_{r,1}(s_1) - \hat{\lambda}(s_1) V^{\pi}_{c,1}(s_1) \\
         &= \min_{y\geq 0} \max_{\pi \in \Delta}  V^\pi_{r,1}(s_1) - y V^\pi_{c,1}(s_1)     \\
        &=  \max_{\pi \in \overline{\mathcal{A}_c}(s_1)}  V^\pi_{r,1}(s_1) & \textrm{(By  \cref{lm:optimality of restricted policies main} )}\\
        &=   V^{\pi^*}_{r,1}(s_1) \\
        &= V^{\pi^*}_{r,1}(s_1) - \hat{\lambda}(s_1) V^{\pi^*}_{c,1}(s_1). 
    \end{align*}
\end{proof}

\subsubsection{Proof of \cref{{lm:regret equivalence}}}
\saddlePtCor*
\begin{proof}

To prove $\text{Regret}(K) = \sum_{k=1}^K V_{r,1}^{\pi^*}(s_1^k) - V_{r,1}^{\pi^k}(s_1^k)$, it suffices to prove 
$ \sum_{k=1}^K V_{r,1}^{\pi^*}(s_1^k)  =  \max_{\pi \in \Dfree }\sum_{k=1}^K V^{\pi}_{r,1}(s_1^k)$. By \cref{lm:optimality of restricted policies main} and under \cref{as:reset-free is feasible.}, we notice that 
$
\max_{\pi \in \Dfree }\sum_{k=1}^K V^{\pi}_{r,1}(s_1^k) = 
\max_{ \pi \in \overline{\mathcal{A}}(s_1^k), \forall k \in [K]}\sum_{k=1}^K V^{\pi}_{r,1}(s_1^k)  
$.
This is equal to $ \sum_{k=1}^K V_{r,1}^{\pi^*}(s_1^k) $ by the definition of $\pi^*$ in the proof of \cref{th:existence of shared saddle-point}.
\end{proof} 
    
\subsubsection{Proof of \cref{lemma:additional saddle pt}}
\lambdaPlus*
\begin{proof}
We prove that \cref{eq:saddle-point} holds for $(\pi^*, \lambda^*)$, that is 
\begin{align*}
          V^{\pi^*}_{r,1}(s_1) - \lambda(s_1) V^{\pi^*}_{c,1}(s_1) 
         \geq  V^{\pi^*}_{r,1}(s_1) - \lambda^*(s_1) V^{\pi^*}_{c,1}(s_1) 
          \geq  V^{\pi}_{r,1}(s_1) - \lambda^*(s_1) V^{\pi}_{c,1}(s_1). 
\end{align*}
Because $V^{\pi^*}_{c,1}(s_1)=0$ (for an initial state $s_1$ such that the CMDP is feasible), the first inequality is trivial:
 \begin{align*}
        V^{\pi^*}_{r,1}(s_1) - \lambda(s_1) V^{\pi^*}_{c,1}(s_1)  
        = V^{\pi^*}_{r,1}(s_1)
        = V^{\pi^*}_{r,1}(s_1) - \lambda^*(s_1) V^{\pi^*}_{c,1}(s_1). 
\end{align*}
For the second inequality, we use \cref{th:existence of shared saddle-point}:
\begin{align*}
    V^{\pi}_{r,1}(s_1) - \lambda^{*}(s_1) V^{\pi}_{c,1}(s_1) 
    \leq  &V^{\pi}_{r,1}(s_1) - \hat{\lambda}(s_1) V^{\pi}_{c,1}(s_1) \\ 
    \leq &V^{\pi^{*}}_{r,1}(s_1) - \hat{\lambda}(s_1) V^{\pi^{*}}_{c,1}(s_1) \\ 
    = &V^{\pi^{*}}_{r,1}(s_1) - \lambda^{*}(s_1) V^{\pi^{*}}_{c,1}(s_1)
\end{align*}
where the first step is because $V^\pi_{c,1}(s_1)$ by definition is in $[0,1]$ and $\lambda^* = \hat{\lambda}+1$, and the second step is by \cref{th:existence of shared saddle-point}.
\end{proof}

\subsubsection{Proof of \cref{th:regret reduction}}
\regretR*
We first establish the following intermediate result that will help us with our decomposition. 
\rdiffLemma*
\begin{proof}
We derive this lemma by \cref{th:existence of shared saddle-point} and \cref{lemma:additional saddle pt}.
First notice by \cref{th:existence of shared saddle-point} and \cref{lemma:additional saddle pt} that for $\lambda' = \lambda^*, \hat{\lambda}$, 
\begin{align*}
    \sum\limits_{k=1}^K \Lagr^k(\pi^*, \lambda') 
    &= \sum\limits_{k=1}^K  V^{\pi^*}_{r,1}(s_1^k) - \lambda'(s_1^k) V^{\pi^*}_{c,1}(s_1^k) \\
    &\leq \sum\limits_{k=1}^K  V^{\pi^*}_{r,1}(s_1^k) - \lambda^k(s_1^k) V^{\pi^*}_{c,1}(s_1^k) = 
    \sum\limits_{k=1}^K \Lagr^k(\pi^*, \lambda^k).
\end{align*}
Then we can derive
\begin{align*}
    \sum\limits_{k=1}^K (\Lagr^k(\pi^*, \lambda') - \Lagr^k(\pi^k, \lambda^k)) &=\sum\limits_{k=1}^K \Lagr^k(\pi^*, \lambda') - \Lagr^k(\pi^*, \lambda^k) + \Lagr^k(\pi^*, \lambda^k) - \Lagr^k(\pi^k, \lambda^k) \\
    &\leq \sum\limits_{k=1}^K \Lagr^k(\pi^*, \lambda^k) - \Lagr^k(\pi^k, \lambda^k) = R_p(\{\pi\}_{k=1}^K, \pi^*)
\end{align*}
which finishes the proof.
\end{proof} 

Then we upper bound Regret($K$) and Resets($K$) by $R_p(\{\pi^k\}_{k=1}^K, \pi_c)$ and $R_d(\{\lambda^k\}_{k=1}^K, \lambda_c)$ for suitable comparators. This decomposition is inspired by the techniques used in \citet{nguyen18}.

We first bound Resets($K$).
\rCV*
\begin{proof}
Notice $\sum_{k=1}^K V_{c,1}^{\pi^k}(s^k_1) = \sum_{k=1}^K \Lagr^k(\pi^k, \hat{\lambda}) - \Lagr^k (\pi^k, \lambda^{*})$ where $(\pi^*,\hat{\lambda})$ is the saddle-point defined in \cref{th:existence of shared saddle-point}. This is because, as defined, $\lambda^* = \hat{\lambda}+1$. Therefore, we bound the RHS. We have
\begin{align*}
\sum\limits_{k=1}^K \Lagr^k(\pi^k, \hat{\lambda}) - \Lagr^k (\pi^k, \lambda^{*}) = &\sum\limits_{k=1}^K \Lagr^k(\pi^k, \hat{\lambda}) - \Lagr^k(\pi^k, \lambda^k) + \Lagr^k(\pi^k, \lambda^k) - \Lagr^k (\pi^k, \lambda^{*})\\ \leq & \sum\limits_{k=1}^K \Lagr^k(\pi^*, \hat{\lambda}) - \Lagr^k(\pi^k, \lambda^k) + \Lagr^k(\pi^k, \lambda^k) - \Lagr^k(\pi^k, \lambda^{*}) \\
\leq & R_p(\{\pi\}_{k=1}^K, \pi^*) + R_d(\{\lambda\}_{k=1}^K, \lambda^{*})
\end{align*}
where second inequality is because $\sum_{k=1}^K \Lagr^k(\pi^*, \hat{\lambda}) \geq \sum_{k=1}^K \Lagr^k(\pi^k, \hat{\lambda})$ by \cref{th:existence of shared saddle-point}, and the first inequality follows from \cref{diffLemma} and \cref{regretXY}. 
\end{proof}
Lastly, we bound Regret$(K)$ with the lemma below and \cref{lm:regret equivalence}.
\rregret*
\begin{proof}
Note that $\Lagr(\pi^*, \lambda^*) = \Lagr(\pi^*, 0)$ since $V_{c,1}^{\pi^*}(s_1^k) = 0$ for all $k \in [K]$. Since by definition, for any $\pi$, $\Lagr^k(\pi, 0) = V_{r,1}^{\pi}(s_1^k)$, we have the following: 
\begin{align*}
&\sum_{k=1}^K V_{r,1}^{\pi^*}(s_1^k) - V_{r,1}^{\pi^k}(s_1^k)= \sum_{k=1}^K \Lagr^k(\pi^*, \lambda^*) - \Lagr^k(\pi^k, 0) \\
    = &\sum_{k=1}^K \Lagr^k(\pi^*, \lambda^*) - \Lagr^k(\pi^k, \lambda^k) + \Lagr^k(\pi^k, \lambda^k) - \Lagr^k(\pi^k, 0)\\
     \leq &R_p(\{\pi\}_{k=1}^K, \pi^*) + R_d(\{\lambda\}_{k=1}^K, 0)
\end{align*}
where the last inequality follows from~\Cref{diffLemma} and \cref{regretXY}. 
\end{proof}

\subsection{Missing Proofs for \cref{sec:linear MDP}} \label{sec:missing proofs of linear mdps}
\subsubsection{Proof of \cref{th:linear mdp special case}}
\rLinearMDP*

We first bound the regret of $\{\pi^k\}_{k=1}^K$ and $\{\lambda^k\}_{k=1}^K$, and then use this to prove the bounds on our algorithm's regret and number of resets with \cref{th:regret reduction}. 

We first bound the regret of $\{\lambda^k\}_{k=1}^K$.
\rregretL*

\begin{proof}
We notice first an equality. 
\begin{align*}
    R_d(\{\lambda^k\}_{k=1}^K, \lambda_c) =&\sum\limits_{k=1}^K \Lagr^k(\pi^k, \lambda^k) - \Lagr^k(\pi^k, \lambda_c)\\    
    = &\sum\limits_{k=1}^K \lambda_c(s_1^k)V^{\pi^k}_{c,1}(s_1^k) - \lambda^k(s_1^k)V^{\pi^k}_{c,1}(s_1^k) \\
    =&\sum\limits_{k=1}^K \lambda_c(s_1^k)V^{\pi^k}_{c,1}(s_1^k) - \lambda^k(s_1^k)V^{\pi^k}_{c,1}(s_1^k) \\ &+\sum\limits_{k=1}^K \lambda_c(s_1^k)V^k_{c,1}(s_1^k)-\lambda_c(s_1^k)V^k_{c,1}(s_1^k) + \lambda^k(s_1^k)V^k_{c,1}(s_1^k) - \lambda^k(s_1^k)V^k_{c,1}(s_1^k)\\
    = &\sum\limits_{k=1}^K  (\lambda^k(s_1^k) - \lambda_c(s_1^k))(- V^k_{c,1}(s_1^k))  + \sum\limits_{k=1}^K   (\lambda^k(s_1^k) - \lambda_c(s_1^k))(V^k_{c,1}(s_1^k) - V^{\pi^k}_{c,1}(s_1^k)). 
\end{align*}

We observe that the first term is an online linear problem for $\theta^k$ (the parameter of $\lambda^k(\cdot)$). In episode $k \in [K]$, $\lambda^k$ is played, and then the loss is revealed. Since the space of $\theta^k$ is convex, we use standard results (Lemma $3.1$ \citep{hazan2016introduction}) to show that updating $\theta^k$ through projected gradient descent results in an upper bound for $\sum_{k=1}^K  (\lambda^k(s_1^k) - \lambda_c(s_1^k))(- V^k_{c,1}(s_1^k))$. We restate the lemma here.

\begin{lemma}[Lemma $3.1$ \citep{hazan2016introduction}]
Let $\mathcal{S} \subseteq \mathbb{R}^d$ be a bounded convex and closed set in Euclidean space. Denote $D$ as an upper bound on the diameter of $\mathcal{S}$, and $G$ as an upper bound on the norm of the subgradients of convex cost functions 
 $f_k$ over $\mathcal{S}$.
Using online projected gradient descent to generate sequence $\{x_k\}_{k=1}^K$ with step sizes $\{\eta_k = \frac{D}{G\sqrt{k}}, k \in [K]\}$
guarantees, for all $K \geq 1$:
\[
\textrm{Regret}_K = \max_{x^* \in \mathcal{K}} \sum_{k=1}^K f_k(x^k) -  f_k(x^*) \leq 1.5 GD\sqrt{K}.
\]
\end{lemma}

Let us bound $D$. By \cref{as:bounded dual variable}, $\lambda^* = \langle \xi(s), \theta^*\rangle$ and $||\theta^*||_2 \leq B$. Since the comparator we use is $\lambda^*$, we can set $D$ to be $B$.
To bound $G$, we observe that the subgradient of our loss function is $\xi(s)V^k_{c,1}(s_1^k)$ for each $k \in [K]$. Therefore, since $V^k_{c,1}(s_1^k)\in [0,1]$ and $||\xi(s)||_2 \leq 1$ by \cref{as:bounded dual variable}, we can set $G$ to be $1$. 
\end{proof}

We now bound the regret of $\{\pi\}_{k=1}^K$.
\rregretPi*
\begin{proof}

First we expand the regret into two terms.
\begin{align*}
    R_p(\{\pi\}_{k=1}^K, \pi_{c})= &\sum\limits_{k=1}^K \Lagr^k(\pi_{c}, \lambda^k) - \Lagr^k(\pi^{k}, \lambda^k)\\
    = &\sum\limits_{k=1}^K V^{\pi_c}_{r,1}(s_1^k) - \lambda^k(s^k_1)V^{\pi_c}_{c,1}(s^k_1)- [V^{\pi^k}_{r,1}(s_1^k) - \lambda^k(s^k_1)V^{\pi^k}_{c,1}(s^k_1)]\\
    = & \sum\limits_{k=1}^K V^{\pi_c}_{r,1}(s_1^k) - \lambda^k(s^k_1)V^{\pi_c}_{c,1}(s^k_1)- [V^{\pi^k}_{r,1}(s_1^k) - \lambda^k(s^k_1)V^{\pi^k}_{c,1}(s^k_1)]\\
    &+ \sum\limits_{k=1}^K [V^k_{r,1}(s_1^k) - \lambda^k(s^k_1)V^k_{c,1}(s^k_1)] - [V^k_{r,1}(s_1^k) - \lambda^k(s^k_1)V^k_{c,1}(s^k_1)]\\
    = &\sum\limits_{k=1}^K V^{\pi_c}_{r,1}(s_1^k) - \lambda^k(s^k_1)V^{\pi_c}_{c,1}(s^k_1) - [V^k_{r,1}(s_1^k) - \lambda^k(s^k_1)V^k_{c,1}(s^k_1)] \\
    &+\sum\limits_{k=1}^K  V^k_{r,1}(s_1^k) - V^{\pi^k}_{r,1} 
 (s_1^k) + \lambda^k(s^k_1)( V^{\pi^k}_{c,1}(s^k_1) - V^k_{c,1}(s^k_1)).
\end{align*}

To bound the first term, we use Lemma $3$ from \citet{ghosh22}, which characterize the property of upper confidence bound. For completeness, we re-write the lemma here. \footnote{Note that \citet{ghosh22} use a utility function rather than a cost function to denote the constraint on the MDP (cost is just $-1 \times $ utility). Also note that their Lemma $3$ is proved for an arbitrary initial state sequence and for any comparator (which includes $\pi^*$).
}

\begin{lemma}[Lemma $3$ \citep{ghosh22}] 
With probability $1-p/2$, it holds that
$\mathcal{T}_1 = \sum_{k=1}^K \Big{(} V^{\pi_c}_{r,1}(s_1^k) - \lambda^k V_{c,1}^{\pi_c}(s_1^k)\Big{)} - \Big{(} V^{k}_{r,1}(s_1^k) - \lambda^k V_{c,1}^{k}(s_1^k)\Big{)} \leq KH\log(|\mathcal{A}|)/\alpha$. Hence, for $\alpha = \dfrac{\log(|\mathcal{A}|)K}{2(1+C+H)}, \mathcal{T}_1 \leq 2H(1 + C + H)$, where $C$ is such that $\lambda^k \leq C$.

\end{lemma}
In our problem setting, we can set $C = B$ in the lemma above. Therefore, the first term is bounded by $2H(1+B+H)$. 
\end{proof}

Lastly, we derive a bound on $R_d(\{\lambda^k\}_{k=1}^K,\lambda_c) + R_p(\{\pi^k\}_{k=1}^K, \pi_c)$, which directly implies our final upper bound on Regret($K$) and Resets($K$) in \cref{th:linear mdp special case} by \cref{th:regret reduction}.  
\begin{lemma}
For any $\pi_c$ and $\lambda_c(s) = \langle \xi(s), \theta_c \rangle$  such that $\| \theta_c \|  \leq B$, we have with probability $1-p$, 
$R_d(\{\lambda^k\}_{k=1}^K,\lambda_c) + R_p(\{\pi^k\}_{k=1}^K, \pi_c) \leq 1.5 B \sqrt{K} + 2H(1+B+H) + O((B+1) \sqrt{d^3H^4K\iota^2})$ where $\iota = \log[\log(|\mathcal{A}|)4dKH/p]$. 
\end{lemma}
\begin{proof}
Combining the upper bounds in \cref{regretLambda} and \cref{regretPi}, we have 
an upper bound of
\begin{align*}
    R_d(\{\lambda^k\}_{k=1}^K,\lambda_c) + R_p(\{\pi^k\}_{k=1}^K, \pi_c) 
    = & 1.5B\sqrt{K} + \sum_{k=1}^K (\lambda^k(s_1^k) - \lambda_c(s_1^k))(V^k_{c,1}(s_1^k) - V^{\pi^k}_{c,1}(s_1^k)) \\
    & + 2H(1+B+H) + \sum_{k=1}^K  V^k_{r,1}(s_1^k) - V^{\pi^k}_{r,1} 
 (s_1^k) + \lambda^k(s^k_1)( V^{\pi^k}_{c,1}(s^k_1) - V^k_{c,1}(s^k_1))
    \\= &1.5 B \sqrt{K} + 2H(1+B+H) + \\
    &+ \sum_{k=1}^K  V^k_{r,1}(s_1^k) - V^{\pi^k}_{r,1}(s_1^k)  + \lambda_c(s_1^k) ( V^{\pi^k}_{c,1}(s_1^k) - V^k_{c,1}(s_1^k) )
\end{align*}
where the last term is the overestimation error due to optimism. To bound this term, we use Lemma $4$ from \citet{ghosh22}. We re-write the lemma here. 
\begin{lemma}[Lemma $4$ \citep{ghosh22}]
WIth probability  at least $1-p/2$, for any $\lambda \in [0, C]$, $\sum_{k=1}^K \Big{(} V^{k}_{r,1}(s_1^k) - V_{r,1}^{\pi^k}(s_1^k)\Big{)} +\lambda\sum_{k=1}^K \Big{(} V^{\pi^k}_{c,1}(s_1^k) - V_{c,1}^{k}(s_1^k)\Big{)} \leq O((\lambda+1)\sqrt{d^3H^4K\iota^2})$ where $\iota = \log[\log(|\mathcal{A}|)4dKH/p]$.
\end{lemma}
Since we have a bound on all $\lambda_c(s_1^k)$ of $B$ for all $k \in [K]$, we have a bound of $O((B+1)\sqrt{d^3H^4K\iota^2})$.
\end{proof}